\documentclass[11pt]{article}
\usepackage{fullpage}

\usepackage[numbers]{natbib}

\usepackage[T1]{fontenc}
\usepackage{lmodern}
\usepackage[utf8]{inputenc} 
\usepackage{hyperref}       
\usepackage{url}            
\usepackage{booktabs}       
\usepackage{nicefrac}       
\usepackage{microtype}      
\usepackage{xcolor}         

\usepackage{amsmath, amsthm, amssymb, amsfonts}
\usepackage{bm}
\usepackage{mathtools}
\usepackage{hyperref}
\usepackage{tabularx}
\usepackage{thm-restate}
\usepackage{subfig}
\usepackage{color}
\usepackage{enumitem}

\usepackage[capitalize]{cleveref}

\usepackage[extdef=true]{delimset}
\usepackage{algorithm}
\usepackage{algorithmic}

\usepackage{mathtools}
\usepackage{thmtools}

\usepackage{microtype}
\usepackage{amssymb}
\usepackage{pifont}
\usepackage{nicefrac}
\usepackage{cancel}
\usepackage{array}
\usepackage{graphicx}
\usepackage{dsfont}
\usepackage{booktabs} 
\usepackage{amssymb}
\usepackage{color}
\usepackage{amsmath}
\usepackage{amsthm}
\usepackage{thmtools}
\usepackage{thm-restate}
\usepackage{algorithm}
\usepackage{algorithmic}
\usepackage{enumitem}
\usepackage{setspace}
\usepackage{hyperref}
\usepackage{nameref}

\Crefname{equation}{Eq.}{Eqs.}
\Crefname{figure}{Fig.}{Figs.}
\Crefname{tabular}{Tab.}{Tabs.}
\Crefname{section}{Sec.}{Secs.}
\Crefname{appendix}{App.}{Apps.}
\Crefname{lemma}{Lem.}{Lems.}
\Crefname{theorem}{Thm.}{Thms.}
\Crefname{remark}{Rem.}{Rems.}
\Crefname{algorithm}{Alg.}{Algs.}
\Crefname{definition}{Def.}{Defs.}

\newtheorem{theorem}{Theorem}
\newtheorem*{theorem*}{Theorem}
\newtheorem*{lemma*}{Lemma}
\newtheorem{lemma}[theorem]{Lemma}

\newtheorem{defn}[theorem]{Definition}
\newtheorem{remark}{Remark}

\newcommand{\R}{{\mathbb R}}

\renewcommand{\P}{{\mathbb P}}

\newcommand{\E}{{\mathbb E}}

\newcommand{\cA}{{\mathcal A}}

\newcommand{\cB}{{\mathcal B}}

\newcommand{\cS}{{\mathcal S}}

\newcommand{\cF}{{\mathcal F}}

\renewcommand{\epsilon}{\varepsilon}
\renewcommand{\hat}{\widehat}

\DeclareMathOperator*{\argmin}{argmin}
\DeclareMathOperator*{\argmax}{argmax}
\newcommand{\indic}{\mathds{1}}

\def \papertitle{One Arrow, Two Kills: An Unified Framework for \\Achieving Optimal Regret Guarantees in Sleeping Bandits}

\newcommand{\red}[1]{\textcolor{red}{#1}}

\makeatletter
\renewcommand{\paragraph}{%
  \@startsection{paragraph}{4}%
  {\z@}{0.1ex \@plus .5ex \@minus .1ex}{-1em}%
  {\normalfont\normalsize\bfseries}%
}
\makeatother
\allowdisplaybreaks

\title{\papertitle}

\author{
Pierre Gaillard \thanks{Univ. Grenoble Alpes, Inria, CNRS, Grenoble INP, LJK, 38000 Grenoble, France. {\tt pierre.gaillard@inria.fr}}
\and
Aadirupa Saha%
\thanks{Toyota Technological Institute at Chicago (TTIC), US; {\tt aadirupa@ttic.edu}}
\and
Soham Dan%
\thanks{IBM Research, US;
{soham.dan@ibm.com (Major part of the work was done while the author was at the University of Pennsylvania)}}
}

\date{}

\begin{document}

\maketitle

\begin{abstract}
{We address the problem of \emph{`Internal Regret'} in \emph{Sleeping Bandits} in the fully adversarial setup, as well as draw connections between different existing notions of sleeping regrets in the multiarmed bandits (MAB) literature and consequently analyze the implications: }  
	Our first contribution is to propose the new notion of \emph{Internal Regret} for sleeping MAB. We then proposed an algorithm that yields sublinear regret in that measure, even for a completely adversarial sequence of losses and availabilities.
	We further show that a low sleeping internal regret always implies a low external regret, and as well as a low policy regret for iid sequence of losses. The main contribution of this work precisely lies in unifying different notions of existing regret in sleeping bandits and understand the implication of one to another. 
	Finally, we also extend our results to the setting of \emph{Dueling Bandits} (DB)--a preference feedback variant of MAB, and proposed a reduction to MAB idea to design a low regret algorithm for sleeping dueling bandits with stochastic preferences and adversarial availabilities.
	%
	%
	The efficacy of our algorithms is justified through empirical evaluations.
\end{abstract}
\vspace{-10pt}

	
\section{Introduction}
\label{sec:intro}

The problem of online sequential decision-making in standard $K$-armed multiarmed bandit (MAB) is well studied in machine learning \cite{Auer00,mohri05} and used to model online decision-making problems 

\begin{figure}[h]
	\centering
		\includegraphics[width=0.5\textwidth]{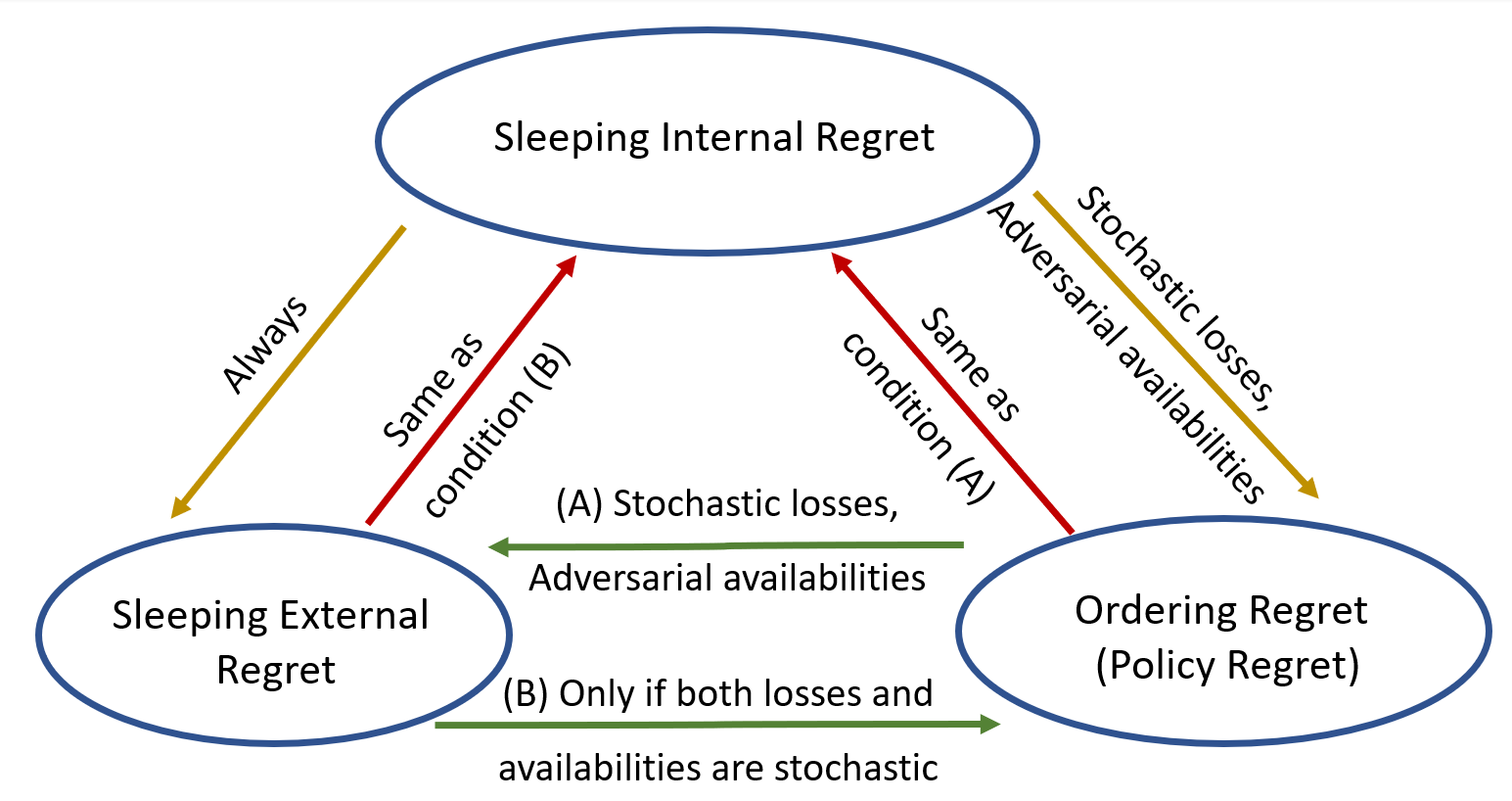}
	\caption{One Arrow, Two Kills: The connections between our proposed notion of Sleeping Internal Regret and different existing notions of regret for sleeping MAB and their implications} 
	\label{fig:rel}
\end{figure}

under uncertainty. 
Due to their implicit exploration-vs-exploitation tradeoff, bandits are able to model clinical trials, movie recommendations, retail management job scheduling etc., where the goal is to keep pulling the `best-item' in hindsight through sequentially querying one item at a time and subsequently observing a noisy reward feedback of the queried arm \cite{Even+06,Auer+02,Auer00,TS12,BubeckNotes+12}.
However, from a practical viewpoint, the decision space (or arm space $\cA = \{1,\ldots,K\}$) often changes over time due to unavailability of some items:  For example, some items might go out of stock in a retail store, some websites could be down, some restaurants might be closed etc. This setting is studied in the multiarmed bandit (MAB) literature as \emph{sleeping bandits} \cite{kanade09,neu14,kanade14,kale16}, where at any round the set $S_t \subseteq \cA$ of available actions could vary stochastically \cite{neu14,cortes+19} or adversarially \cite{kale16,kleinberg+10,kanade14}. 
%
%
 Over the years, several lines of research have been conducted for sleeping multi-armed bandits (MAB) with different notions of regret performance, e.g. policy, ordering, or sleeping external regret \cite{blum2007external,neu14,saha2020improved}.  
 %

In this paper, we introduce a new notion of sleeping regret, called \emph{Sleeping Internal Regret}, that helps to bridge the gaps between different existing notions of sleeping regret in MAB. We show that our regret notion can be applied to the fully adversarial setup, which implies sleeping external regret in the fully adversarial setup (i.e. when both losses and item availabilities are adversarial), as well as policy regret in the stochastic setting (i.e. when losses are stochastic). We further propose an efficient $O(\sqrt{T})$ worst-case regret algorithm for {sleeping internal regret}.  Finally we also motivate the implication of our results for the \emph{Dueling Bandits} (DB) framework, which is an online learning framework that generalizes the standard multiarmed bandit (MAB) \cite{Auer+02} setting for identifying a set of `good' arms from a fixed decision-space (set of items) by querying preference feedback of actively chosen item-pairs \cite{Yue+12,ailon2014reducing,Zoghi+14RUCB,SG19,SGwin18}. The main contributions can be listed as follows: 

\begin{itemize}[leftmargin=0pt, itemindent=10pt, labelwidth=5pt, labelsep=5pt, topsep=0pt,itemsep=2pt]
	\item \textbf{Connecting Existing Sleeping Regret. } The first contribution (\cref{sec:prob}) lies in relating the existing notions of sleeping regret given as:\begin{itemize}[leftmargin=0pt, itemindent=10pt, labelwidth=5pt, labelsep=5pt, topsep=-\parskip,nosep]
		\item The first one, \emph{sleeping external regret}, is mostly used in prediction with expert advice \cite{blum2007external,gaillard2014second}. If the learner had played $j$ instead of $k_t$ at all rounds where $j$ was available, we want the learner to not incur large regret.
		It is well-used to design dynamic regret algorithms~\cite{raj2020non,zhao2020dynamic,campolongo2021closer,zhang2019adaptive,wei2016tracking}. It has the advantage that efficient no-regret algorithms can be designed even when both $S_t$ and losses $\ell_t$ are adversarial. 
		\item The second one, called \emph{ordering regret}, is mostly used in the bandit literature~\cite{kleinberg+10,saha2020improved,kanade2014learning,neu2014online}. It compares the cumulative loss of the learner, with the one of the best ordering $\sigma^*$ that selects the best available action according to $\sigma^*$ at every round. No efficient algorithm exists when both $\ell_t$ and $S_t$ are adversarial: either $S_t$ or $\ell_t$ should be i.i.d \cite{kleinberg+10}. 
		\item We also note that in some works, policies $\pi^*$ (i.e., functions from subsets of $[K]$ to $[K]$) are considered instead of orderings $\sigma^*$, termed as \emph{policy regret} \cite{neu14,saha2020improved}. The latter two are equivalent when the losses are i.i.d., or come from an oblivious adversary with stochastic sleeping. 
	\end{itemize}

	\item \textbf{General Notion of Sleeping Regret. } Our second and one of the primary contribution lies in  introducing a new notion of sleeping regret, called \emph{Internal Sleeping Regret} (\cref{def:int_reg}), which we show actually unifies the different notions of sleeping regret under a general umbrella (see \cref{fig:rel}): We show that (i) Low sleeping internal regret always implies a low sleeping external regret, even under fully adversarial setup. (ii) For stochastic losses is also implies a low ordering regret (equivalently policy regret), even under adversarial availabilities. \emph{Thus we now have a tool, \emph{Sleeping Internal Regret}, optimizing which can simultaneously optimize all the existing notions of sleeping regret (and justifies the title of this work too!)} (\cref{sec:intreg}). 

	\item \textbf{Algorithm Design and Regret Implications. } The main contribution of this works is to propose an efficient algorithm (SI-EXP3, \cref{alg:siexp}) w.r.t. \emph{Sleeping Internal Regret}, and design an $O(\sqrt{T})$ regret algorithm (\cref{thm:internal}). As motivated, the generalizability of our regret further implies $O(\sqrt{T})$ external regret at any setting and also ordering regret for i.i.d losses \cref{rem:reg_implication}. We are the first to achieve this regret unification with only a single algorithm (\cref{sec:alg_internal}). 

	\item \textbf{Extensions: Generalized Regret for \emph{Dueling-Bandits} (DB) and Algorithm. } Another versatility of \emph{Internal Sleeping Regret} is it can be made useful for designing no-regret algorithms for the sleeping dueling bandits (DB) setup, which is a relative feedback based variant of standard MAB  \cite{Zoghi+14RUCB,ailon2014reducing,Busa21survey} (\cref{sec:dueling}). 
	\begin{itemize}[leftmargin=0pt, itemindent=10pt, labelwidth=5pt, labelsep=5pt, topsep=-\parskip-\partopsep,nosep]
		\item \textbf{General Sleeping DB. } Towards this, we propose a new and more unifying notion of sleeping dueling bandits setup that allows the environment to play from different subsets of available dueling pairs ($A_t \subseteq [K]^2$) at each round $t$. This generalizes standard notion of DB setting where $A_t = [K]^2$ without sleeping, but also the setup of Sleeping DB for $A_t = S_t \times S_t$, \cite{saha2021dueling}.
		
		\item \textbf{Unifying Sleeping DB Regret. } Next, taking ques from our notion of \emph{Sleeping Internal Regret} for MAB, we propose a generalized dueling bandit regret, \emph{Internal Sleeping DB Regret} (\cref{eq:regdb}), which unifies the classical dueling bandit regret \cite{Zoghi+14RUCB} as well as sleeping DB regret \cite{saha2021dueling} (\cref{rem:gen_db_reg}).
		
		\item \textbf{Optimal Algorithm Design. } Having established this new notion of sleeping regret in dueling bandits, we propose an efficient and order optimal $O(\sqrt{T})$ sleeping DB algorithm, using a reduction to MAB setup \cite{saha2022versatile} (\cref{thm:dueling}). This improves the regret bound of \cite{saha2021dueling} that only get $O(T^{2/3})$ worst-case regret even in the simpler $A_t = S_t \times S_t$ setting.
	\end{itemize}
	\item \textbf{Experiments. } Finally, in \cref{sec:expts}, we corroborate our theoretical results with extensive empirical evaluation (see \cref{sec:expts}). In particular, our algorithm significantly outperforms baselines as soon as there is dependency between $S_t$ and $\ell_t$. Experiments also seem to show that our algorithm can be used efficiently to converge to Nash equilibria of two-player zero-sum games with sleeping actions (see \cref{rem:nash}).
\end{itemize}


\textbf{Related Works. } 
The problem of regret minimization for stochastic multiarmed bandits (MAB) is widely studied in the online learning literature \citep{Auer+02,TS12,CsabaNotes18,Audibert+10}, 
and as motivated above, the problem of item non-availability in the MAB setting is a practical one, which is studied as the problem of \emph{sleeping MAB}  \cite{kanade09,neu14,kanade14,kale16}, 
for both stochastic rewards and adversarial availabilities \cite{kale16,kleinberg+10,kanade14}
as well as 
 adversarial rewards and stochastic availabilities \cite{kanade09,neu14,cortes+19}. 
In case of stochastic rewards and adversarial availabilities the achievable regret lower bound is known to be $\Omega(\sqrt{KT})$, $K$ being the number of actions in the decision space $\cA = [K]$. The well studied EXP$4$ algorithm does achieve the above optimal regret bound, although it is  computationally inefficient \cite{kleinberg+10,kale16}. The optimal and efficient algorithm for this case is by \cite{saha2020improved}, which is known to yield $\tilde O(\sqrt T)$ regret,\footnote{$\tilde O(\cdot)$ notation hides the logarithmic dependencies.}.

On the other hand over the last decade, the relative feedback variants of stochastic MAB problem has seen a widespread resurgence in the form of the Dueling Bandit problem, where, instead of getting noisy feedback of the reward of the chosen arm, the learner only gets to see a noisy feedback on the pairwise preference of two arms selected by the learner 
\citep{Zoghi+14RUCB,Zoghi+15,Komiyama+15,DTS,ADB,SK21,saha21,SKM21}, or even extending the pairwise preference to subsetwise preferences  \cite{Sui+17,Brost+16,SG18,SGwin18,SGinst20,GS21,Ren+18}. 

Surprisingly, there has been almost no work on dueling bandits in sleeping setup, despite the huge practicality of the problem framework. In a very recent work, \cite{saha2021dueling} attempted the problem of Sleeping DB for the setup of
stochastic preferences and adversarial availabilities, however there proposed algorithms can only yield a suboptimal regret guarantee of $O(T^{2/3})$. 
Our work is the first to achieve $\tilde O(\sqrt T)$ regret for Sleeping Dueling Bandits (see \cref{thm:dueling}).

	\section{Problem Formulation} 
\label{sec:prob}

In this section, we introduce problem of sleeping multiarmed bandit formally, followed by the definition of \emph{Internal Sleeping Regret} -- a new notion of learner's performance in sleeping MAB (\cref{sec:intreg}). The last part of this section discusses the different notions of existing regret bounds in Sleeping MAB (\cref{sec:clsreg}) and their connections (\cref{sec:rel}, summarized in \cref{fig:rel}).

\paragraph{Problem Setting: Sleeping MAB. } Let $[K] = \{1,\dots,K\}$ be a set of arms. At each round $t\geq 1$, a set of available arms $S_t \subseteq [K]$ is revealed to a learner, that is asked to select an arm $k_t \in S_t$, upon which the learner gets to observe the loss $\ell_t(k_t)$ of the selected arm. Note the sequence of item-availabilities $\{S_t\}_{t = 1}^T$ as well as the loss sequence $\{\ell_t\}_{t=1}^T$ can be stochastic or adversarial (oblivious) in nature. We consider the hardest setting of adversarial losses and availabilities, which clearly subsumes the other settings as special cases (see \cref{sec:intreg} for details).

The next thing to understand is how should we evaluate the learner or what is the final objective? Before proceeding to our unifying notion of \emph{Sleeping MAB regret, let us do a quick overview of existing notions of sleeping MAB regret studied in the prior bandit literature.}

\subsection{Existing Objectives for Sleeping MAB} 
\label{sec:clsreg}

\paragraph{1. External Sleeping Regret. } The first notion was introduced by \cite{blum2007external}. Here, the learner is compared with each arm, only on the rounds in which the arm is available:
\begin{equation}
    R_T^{\text{ext}}(k) := \sum_{t=1}^T \big( \ell_t(k_t) - \ell_t(k)\big) \indic\{ k \in S_t \} \,.
    \label{eq:sleeping_regret_ext}
\end{equation}
The learner is asked to control $\max_{k \in [K]} R_T(k) = o(T)$ as $T\to \infty$. In \cite{blum2007external}, the authors provide an algorithm which achieves $R_T(k) \leq O(\sqrt{T})$ for all $k$. 

\paragraph{2. Ordering Regret. } This second notion compares the performance of the learner on all rounds, with any fixed ordering $\sigma = (\sigma_1,\dots,\sigma_K) \in \Sigma$ of the arms, where $\Sigma$ denotes the set of all possible orderings of $[K]$: 
\begin{equation}
    R_T^{\text{ordering}}(\sigma) := \sum_{t=1}^T \ell_t(k_t) - \ell_t\big(\sigma(S_t)\big)\,,
    \label{eq:sleeping_regret_sigma}
\end{equation}
where $\sigma(S_t) = \big\{\sigma_k \text{ s.t. } k = \argmin\{i: \sigma_i \in S_t\}\big\}$ denotes the best arm available in $S_t$. Consequently, in this case, the learner's regret is evaluated is evaluated against the best ordering $\max_{\sigma \in \Sigma} R_T^{\text{ordering}}(\sigma)$. 

It is known that no polynomial time algorithm can achieve a sublinear regret without stochastic assumptions on the losses $\ell_t$ or the availabilities $S_t$, as the problem is known to be NP-hard when both rewards and availabilities are adversarial \cite{kleinberg+10,kanade14,kale16}. For adversarial losses and i.i.d. $S_t$ (where each arm is independently available according to a Bernoulli distribution), \cite{saha2020improved} proposed an algorithm with $O(\sqrt{T})$ regret. For i.i.d. losses and adversarial availabilities, a UCB based algorithm with logarithmic regret was proposed in \cite{kleinberg+10}.

\paragraph{3. Policy Regret} 
A policy $\pi: 2^{[K]} \mapsto [K]$ denotes here a mapping from a set of available actions/experts to an item. Let $\Pi:=\{\pi \mid 2^{[K]} \mapsto [K]\}$ be the class of all policies. In this case, the regret of the learner is measured against a fixed policy $\pi$ is defined as:
\begin{align}
\label{eq:reg}
R_{T}^{\text{policy}}(\pi) =  \E\bigg[ \sum_{t=1}^{T}\ell_t(i_t) - \sum_{t=1}^{T}\ell_t(\pi(S_t)) \bigg],
\end{align}
where the expectation is taken w.r.t. the availabilities and the randomness of the player's strategy \citep{saha2020improved}. As usual, in this case, the learner's regret is evaluated is evaluated against the best policy $\max_{\pi \in \Pi} R_T^{\text{policy}}(\pi)$. 

\subsection{Relations across Different Notions of Existing Sleeping MAB Regret}
\label{sec:rel}
 One may wonder how these above notions are related. Is one stronger than the other? Or does optimizing one implies optimizing the other? Under what assumptions on the sequence of losses $\{\ell_t\}_{t \in [T]}$ and availabilities $\{S_t\}_{t \in [T]}$? We answer all these questions in this section and also summarized them in \cref{fig:rel}. 

\paragraph{1. Relationship between (ii) Ordering Regret and (iii) Policy Regret. } 

These two notions are very close, in principle they are equivalent in all practical contexts where they can be controlled. 
Note for stochastic losses and availabilities, it is easy to see both are equivalent, i.e. $\smash{\max_{\sigma \in \Sigma} R_T^{\text{ordering}}(\sigma)} = \smash{\max_{\pi \in \Pi} R_T^{\text{policy}}(\pi)}$.
 In fact, even when either losses or the availabilities are stochastic, and losses are independent of the availabilities (which are the only settings in which algorithms exist for these notions), we can claim the same equivalence! See \cref{app:pol_vs_ord} for a proof. 
 Thus, \emph{for the rest of this paper, we will only work with Ordering Regret} ($R_T^{\text{ordering}}$). 

\paragraph{2. Relationship between (i) External Sleeping Regret and (ii) Ordering Regret. } 

\begin{itemize}[leftmargin=0pt, itemindent=10pt, labelwidth=5pt, labelsep=5pt, topsep=0pt]
    \item \textbf{Does Ordering Regret \eqref{eq:sleeping_regret_sigma} Implies External Regret \eqref{eq:sleeping_regret_ext}?}
    \begin{itemize}[leftmargin=0pt, itemindent=10pt, labelwidth=5pt, labelsep=5pt, topsep=0pt, nosep]
        \item \textbf{Case (i): Stochastic losses, Adversarial $S_t$: } Yes, in this case it does. Since losses are stochastic, let at any round $t$, $\E[\ell_t(i)] = \mu_i$ for all $i \in [K]$. Then, \vspace*{-10pt}
        \begin{align*}
        \E[R^{\text{ext}}_T(k)] & =\sum_{t=1}^T (\mu_{k_t}-\mu_k) \indic \lbrace k \in S_t \rbrace 
        \\ 
        & \leq \sum_{t=1}^T (\mu_{k_t}-\mu_{k_t^*})\indic \lbrace k \in S_t \rbrace 
        \\
        & \leq \sum_{t=1}^T (\mu_{k_t}-\mu_{k^*})= \E[R^{\text{ordering}}_T(\sigma)]
	    \end{align*}
	    where the first inequality simply follows by the definition of $k_t^* = \sigma^*(S_t)$, $\sigma^*$ being the best ordering in the hindsight, and by noting that for i.i.d. losses for any $i \in S_t$, $\mu_i - \mu_{k_t^*} \geq 0$. 
        \item \textbf{Case (ii): Adversarial losses, Stochastic $S_t$: } The implication does not hold in this case. We can construct examples to show that it is possible to have $\E[R^{ordering}_T(\sigma)]=0$ but $\E[R^{ext}_T(k)]=O(T)$ for some $k \in [K]$ (see \cref{app:prob}). The key observation lies in making the losses $\ell_t$ dependent of availability $S_t$.
    \end{itemize}
    
    \item \textbf{Does External Regret \eqref{eq:sleeping_regret_ext} Implies Ordering Regret \eqref{eq:sleeping_regret_sigma}?} 
    Clearly, this direction is not true in general, as indeed, it would otherwise contradict the hardness result for ordering regret: This is since minimizing ordering regret is known to be NP-Hard for adversarial $\ell_t$ and $S_t$ \cite{kleinberg+10}, while one can easily construct efficient $\tilde{O}(\sqrt T)$ regret algorithms for external regret in the fully adversarial setup, e.g. even our proposed algorithm SI-EXP3 achieves that (see \cref{rem:reg_implication}). Let us analyze in a more case by case basis:  
    \begin{itemize}[leftmargin=0pt, itemindent=10pt, labelwidth=5pt, labelsep=5pt, topsep=-\parskip, nosep]
        \item \textbf{Case (i) Stochastic losses, Adversarial $S_t$:} No! Even for i.i.d. losses the implication does not hold for adversarial sleeping. To see this, we consider the following counter example with three arms ($K=3$). Assume that the arms incur constant losses $\ell_t(k) = k$ when they are available. During the first $T/2$ rounds, we set $S_t = \{1,2\}$ so that the worst arm is unavailable and during the last $T/2$ rounds, the best arm is the one that is sleeping, i.e., $S_t = \{2,3\}$. Then, an algorithm that selects the first arm for $t=1,\dots,T/2$ and the worst arm for $t=T/2+1,\dots,T$ satisfies $R_T^\text{ext}(k) = 0$ for any $k\in [3]$. Yet, $R_T^\text{ordering}\big((1,2,3)\big) = T/2$ because the algorithm chooses the worst arm $3$ instead of $2$ when $S_t=\{2,3\}$. 
        \item \textbf{Case (ii) Adversarial losses, Stochastic $S_t$:} The implication is not true in this case as well. We can simply do the same counter-example by taking i.i.d. availability sets: $S_t = \{1,2\}$ with probability $1/2$ and $S_t = \{2,3\}$ otherwise.  
    \end{itemize}
    This essentially shows ordering regret is a stronger notion of regret compared to external regret.
\end{itemize}

To summarize the above relations, precisely, we present them pictorially in \cref{fig:rel}. However, it is already hard to keep track of the relations of so any different notions of regret! An even more daunting task is, which one to work with? Does optimizing one, necessarily guarantee a low regret in another? we are seeking for a more general sleeping notion that would imply both. To solve this, we introduce thereafter a new notion of sleeping MAB regret that unifies the above notions of regret under a general umbrella.  

\subsection{Internal Sleeping Regret: A New Performance Objective for Sleeping MAB}
\label{sec:intreg}

The notion of \emph{Internal Regret} was introduced  in the theory of repeated games~\cite{foster1999regret}, 
and largely studied in online learning since then, see among other \cite{cesa2006prediction, stoltz2005incomplete, stoltz2005internal, blum2007external}.
Roughly, a small internal regret for some pair $(i,j)$ implies at any round $t$, learner would not have regretted playing $j$, where she actually played arm-$i$ instead. Drawing motivation, for our sleeping MAB setup, we define the notion as follows:

\begin{defn}[Internal Sleeping Regret]
\label{def:int_reg}
For any pair of arms $(i,j) \in [K]^2$, the \emph{internal sleeping regret} is \vspace*{-7pt}
\begin{equation}
	\hspace*{-9pt} R_T^{\text{int}}(i \! \to\! j) \! :=\!\! \sum_{t=1}^T \big( \ell_t(k_t) - \ell_t(j)\big) \indic\{i = k_t, j \in S_t \}. \vspace*{-10pt}
	\label{eq:sleeping_regret_int}
\end{equation}
\end{defn}

Typically, optimizing $R_T^{\text{int}}(i \to j)$ implies, we want that if the learner had played $j$ on all the rounds where he played $i$ and $j$ was available, he does not incur large regret. The strength of this notion is that it can be minimized efficiently (as detailed in \cref{sec:alg_internal}) for general adversarial losses and availabilities which is the key behind our main results (see \cref{rem:reg_implication} and \ref{rem:gen_db_reg}).

	\subsection{Generalizing Power of Internal Sleeping Regret}

In this section, we discuss, how $R_T^{\text{int}}$ generalizes the other existing notions of sleeping regret as discussed in \cref{sec:clsreg}.  

\paragraph{1. Internal Regret vs External Regret} We start by noting that following is a well-known result in the classical online learning setting~\citep{stoltz2005internal} (without sleeping).
 
 \begin{lemma}[Internal Regret Implies External Regret Always\cite{stoltz2005internal}]
\label{lem:ext_int}
For any sequences $(\ell_t)$, $(S_t)$, and  any algorithm, 
$
    R_T^{\text{ext}}(k) = \sum_{i=1}^K R_T^{\text{int}}(i\to k)
$ for all $k \in [K]$.
\end{lemma}

The proof follows from the regret definitions. Thus any uniform upper-bound on the internal regret, implies the same bound for the external regret up to a factor $K$. The other direction is not true though!

\paragraph{2. Internal Sleeping Regret vs Ordering Regret}

\begin{lemma}[Internal Regret Implies Ordering  (for stochastic Losses)] 
\label{lem:implication}
Assume that the losses $(\ell_t)_{t\geq 1}$ are i.i.d..  Then, for any ordering $\sigma$, we have 
\[
    \E\big[R_T^{\text{ordering}}(\sigma)\big] \leq \sum_{i=1}^K \sum_{j \in D_i} \E\big[ R_T^{\text{int}}(i\to j) \big] \,. \vspace*{-5pt}
\]
where $D_i$ is the set of arms such that $\E[\ell_t(j)] \leq \E[\ell_t(i)]$. 
\end{lemma}
Therefore, any algorithm that satisfies $\E\big[ R_T^{\text{int}}(i\to j) \big]  \leq O(\sqrt{T})$, also satisfies $ \E\big[R_T^{\text{ordering}}(\sigma)\big] \leq O(\sqrt{T})$. 
The proof is deferred to the \cref{app:int_to_ord}.

\begin{remark}
\label{rem:adl_stcs}
An interesting research direction for the future would be to 
see if the sleeping internal regret also implies the ordering regret for adversarial losses and stochastic availabilities?  {Our experiments seem to point into this direction} (see \cref{app:expts}), but despite our efforts, we could not prove it. Such a result would in particular imply that any algorithm that can achieve sublinear regret w.r.t. sleeping internal regret $R_T^{\text{int}}$ (we in fact proposed such an algorithm in \cref{sec:alg_internal}, see SI-EXP3), would actually satisfy a best-of-both worlds guarantee! That is if either the losses or the availabilities are stochastic, the algorithm it will in turn incur a sublinear regret w.r.t. ordering regret $R_T^{\text{ordering}}$ as well. 
%
\end{remark}

\begin{remark}
On the other hand, it is well known that for adversarial losses, a small external regret does not imply a small internal regret~\cite{stoltz2005internal} even when $S_t = [K]$. But, when losses are stochastic and availabilities are adversarial, minimizing the ordering regret does control the internal regret (see \cref{app:int_to_ord}. 
\end{remark}

	\section{SI-EXP3: An Algorithm for Minimizing Internal Sleeping Regret} 
\label{sec:alg_internal}
We now consider the problem setting of Sleeping MAB \cref{sec:prob} and propose an EXP3 based algorithm that is shown to yield sublinear sleeping regret guarantee. It is worth noting that, our algorithm applies to the hardest setting of adversarial losses and availabilities, which clearly subsumes the stochastic settings (losses or availabilities) as a special case: As proved in \cref{thm:internal}, our proposed algorithm SI-EXP3 achieves $\tilde O(\sqrt{KT})$ internal regret for any arbitrary sequence of losses $\{\ell_t\}_{t \in [T]}$ and availabilities $\{S_t\}_{t \in [T]}$. 
 Further the generalizability of our internal regret (see \cref{fig:rel}) implies $O(\sqrt{T})$ external regret at any setting and also ordering regret for i.i.d losses, as detailed in \cref{rem:reg_implication}.

Our regret analysis is inspired from the construction of~\cite{stoltz2005incomplete} (see Section 3, Thm. 3.2) for the internal regret although the `sleeping component' or item non-avilabilities is not considered. Another relevant work is~\cite{blum2007external}, which designs an algorithm minimizing a variant of internal sleeping regret with a subtle difference: it considers time selection functions $I \in \mathcal{I} \subseteq \{0,1\}^T$ instead of sleeping actions. More precisely, the regret considered by \cite{blum2007external} is of the form\vspace*{-7pt}
$$
    \qquad \max_{I \in \mathcal{I}} \sum_{t=1}^T \big(\ell_t(k_t) - \ell_t(j)\big) \indic\{k_t = i, j \in I(t)\}\,.
$$
Thus $R_T^{\text{int}}(i\to j)$ would correspond to the choice $I(t) = \indic\{j \in S_t\}$, but the dependence on the arm $j$ is not possible in their definition and makes the adaptation of their algorithm challenging. Furthermore, their regret bound (Thm. 18) only holds in the full information setting. It may be adapted to bandit feedback, but would yield a suboptimal regret $\smash{O(K \sqrt{TK \log K})}$ (which is the internal regret upper-bound they obtain in Thm. 11 without the sleeping component) in comparison with \cref{thm:internal}. 
We now describe our main algorithm, SI-EXP3, for optimizing Internal Sleeping Regret.


\subsection{Algorithm: SI-EXP3} 
The Sleeping-Internal-EXP3 (SI-EXP3) procedure is a two-level algorithm. At round $t\geq 1$, the master algorithm forms a probability vector $p_t \in \Delta_K$ over the arms, which is used to sample the played action $k_t \sim p_t$. The vector $p_t$ is such that $p_t(i) = 0$ for any $i \notin S_t$. A subroutine, based on EXP3~\cite{Auer+02}, combines $K(K-1)$ sleeping experts indexed by $i \to j$, for $i \neq j$. Each expert aims to minimize the internal sleeping regret $R_T^{\text{int}}(i\to j)$. We detail below how to construct $p_t$. 


\begin{algorithm}[h]
	\caption{SI-EXP3: Sleeping Internal Regret Algorithm for MAB}
	\label{alg:siexp}
	\begin{algorithmic}[1]	
		\STATE {\bfseries input:} Arm set: $[K]$, learning rate $\eta > 0$
            \STATE {\bfseries init:} $E := \{(i,j) \in [K]^2, i\neq j\}$\\
            \phantom{{\bfseries init: }}$\tilde q_1 \in \Delta_E$ uniform distribution on $E$
		\FOR{$t = 1, 2, \ldots, T$}
            \STATE Observe $S_t \subseteq [K]$ and define $q_t \in \Delta_E$ as in~\eqref{eq:defqt}
		\STATE Define $p_t \in \Delta_K$ by solving~\eqref{eq:defpt}
            \STATE Predict $k_t \sim p_t$ and observe $\ell_t(k_t)$
            \STATE Define $\hat \ell_t(k) = \frac{\ell_t(k)}{p_t(k)} \indic\{k = k_t\}$ for all $k \in [K]$
            \FOR{$(i,j) \in E$}
            \STATE Define $p_t^{i \to j} \in \Delta_K$ as in~\eqref{eq:defptijk}
            \STATE Define $\hat \ell_t(i\to j)$ as in~\eqref{eq:defhatellij}
            \ENDFOR
		\STATE Update $\tilde q_{t+1}(i \to j) \propto \tilde q_t(i \to j) e^{- \eta \hat \ell_t(i\to j)}$
		\ENDFOR  
	\end{algorithmic} 
\end{algorithm} 

For any $i\neq j$, we denote by $p_t^{i\to j} \in \Delta_K$ the probability vector that moves the weight of  $p_t$ from $i$ to $j$, \vspace*{-5pt} 
\begin{equation}
    \label{eq:defptijk}
    p_t^{i\to j}(k) = \left\{
    \begin{array}{cc}
        0 & \text{if } k= i \\
        p_t(i) + p_t(j) & \text{if } k = j \\
        p_t(k) & \text{otherwise}
    \end{array} \right. \,.
\end{equation}
As usually considered in adversarial multi-armed bandits, for any active arm $i \in S_t$, we define the associated estimated loss $\hat \ell_t(i) := \ell_t(i)  \indic\{ i = k_t \}/p_t(i)$. 
Furthermore, by abuse of notation, we also define for any $i,j \in [K]$, $i\neq j$ the loss \vspace*{-5pt}
\begin{equation}
    \label{eq:defhatellij}
    \hat \ell_t\big( i \to j \big) := \left\{ 
    \begin{array}{ll}
        \sum_{k=1}^K p_t^{i \to j}(k) \hat \ell_t(k) & \text{if } j \in S_t \\
         \ell_t(k_t) & \text{otherwise}
    \end{array} \right.
\end{equation}
The subroutine then computes the exponential weighted average of the experts $i \to j$, by forming the weights \vspace{-7pt}
\begin{equation}
    \hspace*{-10pt} \tilde q_t(i\to j) := \frac{\exp\Big( - \eta \sum_{s=1}^{t-1}  \hat \ell_s (i\to j) \Big)}{\sum_{i'\neq j '}\exp\Big( - \eta \sum_{s=1}^{t-1}  \hat \ell_s(i' \to j') \Big) } \,.
    \label{eq:deftildeqt}
\end{equation}
To avoid assigning mass from an active item $i$ to an inactive $j \notin S_t$, the subroutine then normalizes those weights so that sleeping experts get $0$ mass
\begin{equation}
    \label{eq:defqt}
     q_t(i \to j) := \frac{\tilde q_t(i \to j) \indic\{j \in S_t\} }{\sum_{i'\neq j'} \tilde q_t(i'\to j') \indic\{j' \in S_t\}} \,.
\end{equation}
Finally, the master algorithms forms $p_t \in \Delta_K$ by solving the linear system
\begin{equation}
    \label{eq:defpt}
    p_t = \sum_{i \neq j} p_t^{i \to j} q_t(i \to j) \,.
\end{equation}
The existence and the practical computation of such a $p_t$ is an application of Lemma 3.1 of~\cite{stoltz2005incomplete}.

\subsection{Regret Analysis of SI-EXP3} 
We now analyze the sleeping internal regret guarantee of \cref{alg:siexp} (\cref{thm:internal}), and also the implications to other notions of sleeping regret (\cref{rem:reg_implication}).

\begin{theorem} 
\label{thm:internal}
Consider the problem of Sleeping MAB for arbitrary (adversarial) sequences of losses $\{\ell_t\}$ and availabilities $\{S_t\}$. Let $T \geq 1$ and $\smash{\eta^2 = (\log K)/ \big(2 \sum_{t=1}^T |S_t|\big)}$. Assume that $0\leq \ell_t(i)\leq 1$ for any $i \in S_t$ and $t \in [T]$. Then, \vspace*{-7pt}
\[
    \E\big[R_T^{\text{int}}(i\to j)\big] \leq 2  \sqrt{2 \log K \sum_{t=1}^T |S_t|} \leq 2 \sqrt{2 T K \log K} \,,
\]
for all $i \neq j$ in $[K]$.
\end{theorem}

The proof is postponed to \cref{app:proof_thm}. The learning rate $\eta$ can be calibrated online at the price of a small multiplicative constant by choosing a time-varying learning rate $\smash{\eta_t^2 = (\log K)/ \big(2 \sum_{s=1}^t |S_s|\big)}$ or by using the doubling-trick technique~\cite{cesa2006prediction}.

\begin{remark}
	\label{rem:reg_implication}
 The above theorem also implies a bound of order $O(K\sqrt{KT \log K})$ for the external sleeping regret by~\cref{lem:ext_int} and $O(K^2 \sqrt{KT \log K})$ for the ordering regret when the losses are i.i.d. by~\cref{lem:implication}. SI-EXP3 is the first to simultaneously achieve order-optimal \emph{sleeping external regret} for fully adversarial setup as well as \emph{ordering regret} for stochastic losses and this was possible owning to the versatility of \emph{Sleeping Internal Regret} as summarized in \cref{fig:rel}.
\end{remark}

	\section{Implications: Better Regret for Sleeping Dueling Bandits}
\label{sec:dueling} 

In this section, we show the implication of our result to sleeping dueling bandits.

\textbf{Motivation behind a generalized DB objective. } We show interesting use-cases of our generalization such as when the user is asked at each round to choose two different actions $i \neq j$. Note that in the dueling bandit literature, the user may choose replicated arms $(i,i)$, and is expected to converge to an optimal pair $(i^*,i^*)$. However, in many applications, this does not make sense to show users the same pair of items $(i,i)$, rather it might be preferred to see their top-two choices, i.e we would expect the algorithm to converge to the best pair $(i,j)$, $i\neq j$ (more motivating examples are provided after \cref{rem:gen_db_reg}). Classical dueling bandit algorithms do not easily allow for such a restriction, whereas this can be easily achieved with our sleeping procedure.

\subsection{Our Problem Setting: Sleeping Dueling Bandits (Sleeping DB)}
\label{sec:si_db}

We generalize the setting of for dueling bandits with adversarial sleeping  of~\cite{saha2021dueling}. We consider the stochastic dueling bandit setting with a preference matrix $P \in [0,1]^{K\times K}$ such that $P(i,j) = 1 - P(i,j)$ is the probability of item $i$ to beat item $j$ in some round. Furthermore, we assume that their exists a total ordering of the arms $\sigma$, such that $P(\sigma(i),j) \geq  P(\sigma(i'),j)$ for all $\sigma(i) \leq \sigma(i')$ and all $j \in [K]$. This is for instance satisfied for utility-based preference matrices \cite{ailon2014reducing}, or for Plackett-Luce model \cite{azari2012random}, where it is assumed that the $K$ items are associated to positive
score parameters $\theta_1,\dots,\theta_K$ and $P(i, j) = \theta_i/(\theta_i+\theta_j)$ for all $i, j \in [K]$.

Before each round $t\geq 1$, an adversary reveals a set of possible dueling pairs $A_t \subseteq [K]^2$. We assume that if $(i,j) \in A_t$ then $(j,i) \in A_t$. Furthermore, we denote for any $i \in [K]$ by 
$
A_t(i) := \{j \in [K], (i,j) \in A_t\},
$ 
the set of possible adversaries for $i \in [K]$, and as usual by $S_t := \{i \in [K], A_t(i) \neq \emptyset\}$ the set of available arms at time $t$. 
After observing $A_t$, the learner selects a pair of items $(i_t,j_t) \in A_t$ and observes the result of the duel $o_t(i_t,j_t)$ which follows a Bernoulli distribution with mean $P(i_t,j_t)$. 

\textbf{Performance: Internal Sleeping DB regret.} We measure the learner's regret w.r.t. the following regret measure: \vspace*{-5pt}
\begin{equation}
\label{eq:regdb}
   R_T^{\texttt{SI-DB}} \!\! =  \frac{1}{2} \sum_{t=1}^T \Big( \underset{j^* \in A_t(i_t)}{\max} \!\! P(j^*, i_t) + \!\! \underset{i^* \in A_t(j_t)}{\max} \!\! P(i^*, j_t) - 1 \Big) \,.
\end{equation}
Since the definition is inspired from internal regret, we term it as Internal Sleeping DB regret (or \texttt{SI-DB} in short) --- the measure essentially evaluates the dueling choices of the learner $(i_t,j_t)$ against their best `available competitor' according to $A_t(\cdot)$. 

\begin{remark}[Generalizability of $R_T^{\texttt{SI-DB}}$]
\label{rem:gen_db_reg}
It is noteworthy that if all pairs are available $A_t = [K]^2$, then $i_t^* = j_t^*$ is the Condorcet Winner (CW) (see \cite{Zoghi+14RUCB} for definition) for all rounds since $\P$ respects total ordering. Thus, in this case, $R_T^{\texttt{SI-DB}}$ reduces to the standard CW-regret studied in DB~\cite{Yue+09,Zoghi+14RUCB,Busa21survey}. Moreover, $R_T^{\texttt{SI-DB}}$ also generalizes the notion of Sleeping Dueling Bandit of \cite{saha2021dueling} for the special case $A_t = S_t \times S_t$ (i.e. when all pairs of the available items are feasible): This is since for this case we again have $i_t^* = j_t^*$ (owing to the total ordering assumption of $\P$), and hence we can recover their notion of sleeping regret (see Eqn. (1) of \cite{saha2021dueling}). Nevertheless, our new notion offers more flexibility as we now show with some application examples. 
\end{remark}

\textbf{Motivating Examples: Practicability of $R_T^{\texttt{SI-DB}}$ . }\\[2pt]
\textbf{ (i.) Dueling bandits with non-repeating arms.} A first example consists in choosing $A_t = \{(i,j) \in [K]^2, i \neq j\}$. Our algorithm will be forced to play two different items at every round. Such a restriction is new to dueling bandits although it makes sense in many applications, such as recommendation systems in which we may want to suggest pairs of different items. Our algorithm will converge to the best pair. An interesting question for future work is to generalize our strategy to any size $M$ (possibly larger than 2) of subsets of unique battling items. A similar setting was considered by \cite{saha2018battle} but they allow choosing replicate items.
\\[2pt]
\textbf{ (ii). Preference learning with categories.} Another example comes from an application in which the arms may be grouped into different categories (or teams). For example, one may thing of a recommendation system for movies. The latter could be action movies, documentaries, TV series, or romantic movies. The system could be asked to suggest at every round movies from different categories to provide diversity into the suggestion. Our algorithm would simultaneously learn the best movies but also the best two categories. In addition, the possibility of sleeping allows the film collection to vary over time. 

\subsection{Sparring SI-EXP3: An Algorithm for Sleeping DB and Regret Analysis}
Following the generic reduction from multi-armed bandit to dueling bandit from~\cite{saha2022versatile} (see Section $4$), we consider the following algorithm. We run two versions $\cA^{\text{left}}$ and $\cA^{\text{right}}$ of the internal sleeping regret algorithm of \cref{thm:internal} in parallel.  
At each round $t$, $i_t$ is chosen by $\cA^{\text{left}}$, which is run on the availability sets $S_t$ and losses $\ell_t^{\text{left}}(k) = o_t(j_t, k)$, $k \in S_t$. After $i_t$ is chosen, $\cA^{\text{right}}$ chooses $j_t$, by using the availability sets $A_t(i_t)$ and losses $\ell_{t}^{\text{right}}(k) = o_t(i_t, k)$. 
We call the algorithm as \emph{Sparring SI-EXP3}, following the classical nomenclature from \cite{ailon2014reducing} which first invented the idea of designing a DB algorithm by making two MAB algorithms competing against another, and famously termed it as `\emph{Sparring}'.

\begin{theorem} \label{thm:dueling}
Consider the problem setting of Sleeping DB defined above (\cref{sec:dueling}) and let $T \geq 1$. Then, Sparring SI-EXP3 satisfies \vspace*{-5pt}
\[
    \E[R_T^{\texttt{SI-DB}}] \leq 2 K^2 \sqrt{2 T K \log K} \,.
\]
\end{theorem}


The proof is postponed to \cref{app:db}. As explained in \cref{rem:gen_db_reg}, by choosing $A_t$ of the form $S_t \times S_t$ for some subset $S_t \subseteq [K]$, we retrieve the setting of~\cite{saha2021dueling}. Note that they provide distribution dependent upper-bounds while we present worst-case upper-bound. They show a high-probability regret bound of order $O(K^2 \log(1/\delta) /\Delta^2)$ for an UCB based algorithm, and a $O(K^3/\Delta^2 + K^2 \log(T)/\Delta)$ upper-bound on the expected regret of an algorithm based on empirical divergences. Their analysis are quite technical and non-trivial to adapt to general sets $A_t$ as our result. Furthermore , both their algorithms yield a worst-case regret of order $\smash{O(T^{2/3})}$ while we only suffer $\smash{O(\sqrt{T})}$.



	\section{Experiments}
\label{sec:expts}

\begin{figure*}[!t]
    \centering
    \begin{minipage}{.3\textwidth}
        \includegraphics[width=\textwidth]{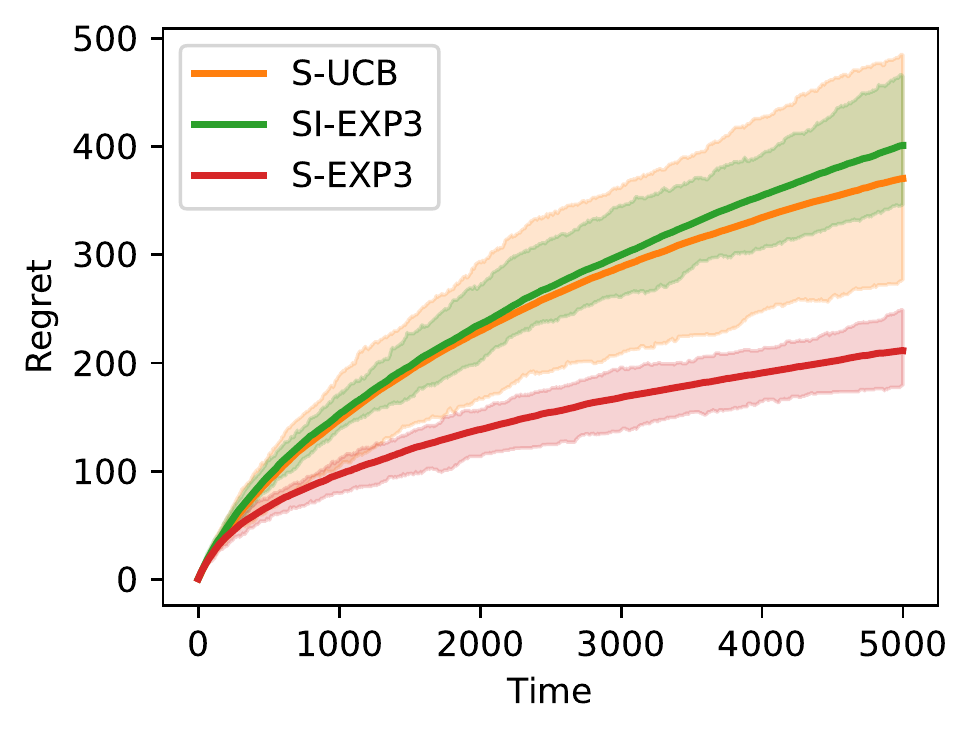}
    \end{minipage}\quad 
    \begin{minipage}{.3\textwidth}
        \includegraphics[width=\textwidth]{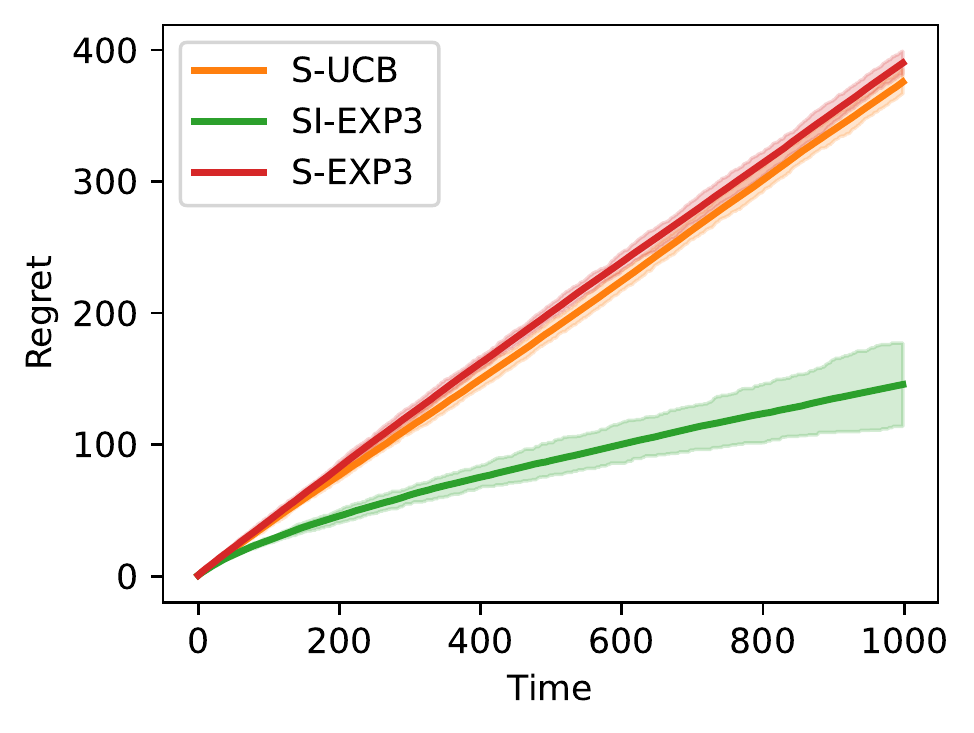}
    \end{minipage}    
    \begin{minipage}{.3\textwidth}
        \includegraphics[width=\textwidth]{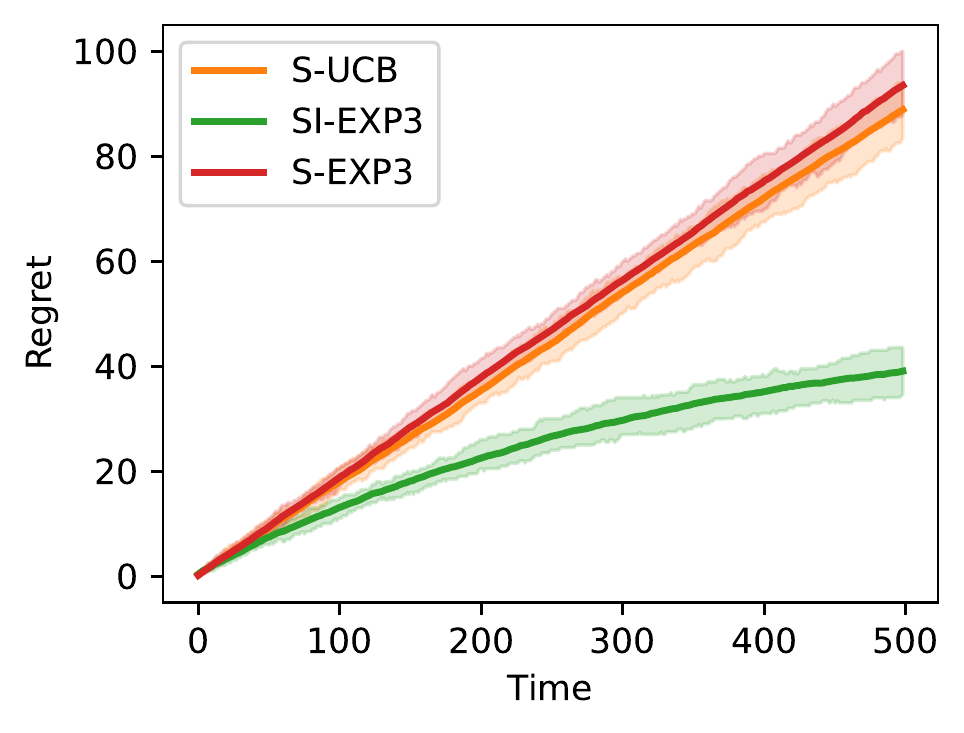}
    \end{minipage}
    \caption{[Left] Stochastic environment [Middle] Dependent environment [Right] Rock Paper Scissors}
    \label{fig:exp1}
\end{figure*}


We provide synthetic experiments in sleeping multi-armed bandits. In all the experiments, we plot the policy regret $R_T^{policy}$ for MAB~\eqref{eq:reg}. All experiments are averaged across $50$ runs. Further experiments (including some in the dueling setup) are provided in \cref{app:expts}.
We compare the results of the following algorithms: 
\begin{itemize}[nosep,leftmargin=*,topsep=-\parskip]
    \item SI-EXP3: \cref{alg:siexp}; 
    \item S-UCB: A sleeping UCB procedure~\cite{kleinberg+10} for ordering regret with stochastic losses;
    \item S-EXP3: Sleeping-EXP3G ~\cite{saha2020improved} for ordering regret with adversarial losses and i.i.d. sleeping.
\end{itemize}
We compare to these two algorithms because they achieve state-of-the-art performance in their respective settings. The hyper-parameters $\eta$ of SI-EXP3 and $(\eta,\lambda)$ of S-EXP3 are considered as time-varying hyper-parameters and set to $t^{-1/2}$.

\paragraph{Stochastic environment.} The losses and availabilies for $K = 10$ arms are i.i.d. and respectively follow Bernoulli distributions with mean $\mu_k$ and $a_k$. The latter are uniformly sampled at the start of each run on $(0,1)$. Rounds with no available arms are skipped. 

\paragraph{Dependent environment.} We consider the following semi-stochastic environment with $K=3$. The pairs $(S_t,\ell_t)$ are still i.i.d. but the losses $\ell_t$ depend on the availabilities. The sets $S_t$ are first uniformly sampled among $\{1,2\},\{1,2,3\},\{1,3\}$ and $\{2,3\}$. According to the values of $S_t$, the loss vectors are then respectively $(0,.5,x),(0,.5,1),(1,x,0),$ and $(x,0,1)$, where $x$ means that the arm is sleeping. 

\paragraph{Rock-Paper-Scissors.} We consider a repeated two-player zero-sum game with payoff matrix \vspace*{-5pt}
\[
    P = \left(
    \begin{smallmatrix}
        0 & 1 & -1 \\
        -1 & 0 & 1 \\
        1 & -1 & 0
    \end{smallmatrix}
    \right) \,.\vspace*{-5pt}
\]
We assume that at the start of each round some action may be unavailable ($S_t$ is uniformly sampled as in the previous environment). The game is then played on $S_t$ only. We consider an opponent that is playing the Nash equilibrium of the sub-games (i.e., the game with the payoff matrix $P$ restricted to $S_t$) and run each algorithm against that opponent. 

\paragraph{Results.}
The cumulative regrets are provided in \cref{fig:exp1}. We find that as soon as there are dependencies between the loss vectors and availabilities, SI-EXP3 significantly outperforms the other two algorithms. This is not surprising: S-EXP3 and S-UCB were indeed designed to perform well with respect to the best fixed ordering. Typically, they first order the actions with respect to their average performance on all rounds, and play the best action that is available in $S_t$. In case of dependency between $S_t$ and $\ell_t$, the best ordering may vary over the rounds and S-UCB and S-EXP3 incur linear regret. Note that this situation can happen often in real life. An example is an internet sales site that wants to offer products. Some products are useless if others are out of stock. For instance, it is less interesting to offer the products of a cooking recipe if some of the ingredients are not available. 

\begin{remark}[Application to game theory with sleeping actions] \label{rem:nash}
For sleeping two-player zero-sum games, the best policy to play depends on the actions available to the opponent: if Scissors is not available, then Paper is the best action, although when all actions are available the optimal policy is $(1/3,1/3,1/3)$. For instance, the Nash equilibrium of $P$ restricted to $S_t = \{1,2\}$ (Rock, Paper), is $(0,1,0)$ (i.e., play Paper). Yet, here, all actions are on average equally good (i.e., taking the expectation over $S_t$); and S-UCB and S-EXP3 will converge to $(1/2,1/2,0)$ when Scissor is unavailable and incur linear regret. On the other hand, SI-EXP3 is able to leverage the dependence between $S_t$ and $\ell_t$ and choose the right action. In \cref{app:expts}, we provide an additional experiment with two-player zero-sum randomized games (with a random payoff matrix $P$). An intriguing question for future work is whether SI-EXP3 converges to the Nash equilibrium of each subgame (or whether it obtains sublinear regret against an adversary that plays the Nash of $P$ restricted to actions of $S_t$).  
\end{remark}

	\paragraph{Perspectives}
\label{sec:conclusion}
%
%
%
On a high level, the general theme of this work--to unify different notions of performance measure under a common umbrella and designing efficient algorithms for the general measure--can be applied to several other bandits/online learning/learning theory settings, which opens plethora of new directions. 

Specifically as an extension to this work, some of the interesting open challenges could be: 
\textbf{(i).} to understand if sleeping internal regret also implies the ordering regret for adversarial losses but stochastic availabilities (see \cref{rem:adl_stcs}). 
\textbf{(ii).} to derive gap dependent bounds sleeping dueling bandit regret for stochastic preferences and adversarial sleeping, same as derived for its MAB counterpart in \cite{kleinberg+10} or in a recent work \cite{saha2021dueling} which though only gave suboptimal regret guarantees? 
\textbf{(iii).} to understand if our results can be extended to the subsetwise generalization of dueling bandits, studied as the \emph{Battling Bandits} \cite{SG19,SGinst20,Ren+18}; amongst many. 

	\newpage
	
	\bibliographystyle{plainnat}
	\bibliography{refs,bib_bdb}
	
	\newpage
	\appendix
\onecolumn{
\section*{\centering\Large{Supplementary: \papertitle}}
\vspace*{1cm}
	
\section{Appendix for \cref{sec:prob}}
\label{app:prob}

\subsection{Low ordering regret $R_T^{\text{ordering}}$ with Adversarial losses, Stochastic availabilities does not imply low external regret $R_T^{\text{ext}}$}

\label{app:prob1}

\begin{lemma}
 There exists a sequence of i.i.d. availabilities $(S_t)_{t\geq 1}$ and a sequence of losses $(\ell_t)_{t\geq 1}$ (possibly depending on $S_t$), such that, there exists an algorithm with 
 \[
    \max_{\sigma} R_T^{\text{ordering}}(\sigma) = o(T) \qquad \text{and}\qquad \max_k R_T^\text{ext}(k) = \Omega(T) \,,
 \] 
 as $T\to \infty$. 
\end{lemma}

\begin{proof} We provide the following example.
Consider a MAB problem with $3$ arms, $K=3$. Suppose the problem encounters the following availability sets $ \cA_1=\{1,2,3\}, \cA_2= \{1,2\},, \cA_3=\{1,3\}, \cA_4=\{2,3\}$ uniformly, i.e. $\P(\cS_t = \cA_i) = 1/4$ for all $i \in [4]$ and $t \in [T]$, where $\cS_t$ being the availability set at time $t$. Further let us consider the adversarial (rather set dependent) loss sequence generated as follows:
\[
\begin{array}{l|lll}
        & \ell_t(1) & \ell_t(2) & \ell_t(3) \\ \hline 
    \text{if } \cS_t = \cA_1 & 0 & 1 & 1\\
    \text{if } \cS_t = \cA_2 &  0 & 1 &  x \\
    \text{if } \cS_t = \cA_3& 1 & x  & 0\\
    \text{if } \cS_t = \cA_4& x&  1 & 1
\end{array}
\]
where $x$ can be any arbitrary loss value. For this example, the best orderings are $(1,2,3), (1,3,2)$ and $(3,1,2)$, that get a cumulative loss equals to $T/2$ in expectation. Indeed,
\[
   \frac{4}{T} \E\left[ \sum_{t=1}^T \ell_t(\sigma(S_t)) \right] = \left\{ 
        \begin{array}{ll}
            2 & \text{ if } \sigma = (1,2,3) \\
            2  & \text{ if } \sigma = (1,3,2) \\
            4 & \text{ if } \sigma = (2,1,3) \\
            3 & \text{ if } \sigma = (2,3,1) \\
            2  & \text{ if } \sigma = (3,1,2) \\
            3 & \text{ if } \sigma = (3,2,1) 
        \end{array}\right. \,.
\]
Consider and algorithm that plays $k_t = \sigma(S_t)$ according to the ordering $\sigma = (3,1,2)$. Then, $\E\big[\sum_{t=1}^T \ell_t(k_t)\big] = T/2$. It has thus no-regret $R_T^{\text{ordering}}(\sigma^*) \leq 0$ with respect to any ordering $\sigma^*$. Yet, its internal regret with respect to action 1 is
\[
    R_T^{\text{ext}}(1) = \E\left[\sum_{t=1}^T (\ell_t(k_t) - \ell_t(1)) \indic\{1\in S_t\} \right] = \frac{T}{4}.
\]
This implies a `no-regret ordering regret learner' does not imply a `no-regret external regret learner' for any arbitrary sequence of adversarial losses, stochastic availabilities.
\end{proof}

\subsection{Low ordering regret $R_T^{\text{ordering}}$ with Stochastic losses, Adversarial availabilities does imply low internal regret $R_T^{\text{int}}$}

\begin{lemma}
Let $(\ell_t)_{t\geq 1}$ be an i.i.d. sequence of losses. Then, for any sequence of availability sets $(S_t)_{t\geq 1}$ such that $S_t$ may only depend on $(\ell_s)_{s\leq t-1}$
\[
    \max_{1\leq i,j\leq K} R_T^\text{int}(i\to j) \leq \max_\sigma R_T^\text{ordering}(\sigma) \,,
\]
for any algorithm.
\end{lemma}

\begin{proof}
Consider stochastic losses such that $\ell_t(k)$ are i.i.d. with mean $\mu_k$ for all $k \in [K]$, and any sequence of availability sets $S_1,\dots, S_T \subseteq [K]$ (that can only depend on information up to $t-1$). Let $\sigma^*$ be an optimal ordering
\[
        \sigma^* \in \argmin_{\sigma} \E\left[ \sum_{t=1}^T \ell_t(\sigma(S_t)) \right] \,.
\]
Then, for all $S\subseteq [K]$, $\mu_{\sigma^*(S)} = \min_{i \in S} \mu_i$. Let $k_t$ be the predictions of any algorithm.
Let $(i,j) \in [K]^2$. Then,
\begin{align*}
    R_T^{\text{int}}(i\to j) 
        & = \E\left[ \sum_{t=1}^T (\ell_t(i) - \ell_t(j)) \indic\{i = k_t, j \in S_t\} \right] \\
        & \leq  \E\left[ \sum_{t=1}^T (\mu_{k_t} - \mu_j) \indic\{i = k_t, j \in S_t\} + (\mu_{k_t} - \mu_{\sigma^*(S_t)}) \indic\{k_t \neq i \text{ or } j \notin S_t\} \right] \\
        & \leq \E\left[ \sum_{t=1}^T \mu_{k_t} - \mu_{\sigma^*(S_t)}  \right] = R_T^{\text{ordering}}(\sigma^*)\,,
\end{align*}
where the inequalities are because $\mu_{\sigma^*(S_t)} \leq \mu_i$ for any $i \in S_t$. This concludes the proof.
\end{proof}

\subsection{Equivalence of Policy and Ordering Regret}
\label{app:pol_vs_ord}

The policy regret is a stronger notion than ordering regret in general.
From their definitions, we see
\[
        \max_{\sigma} R_T^{\text{ordering}} (\sigma) \leq \max_{\pi} R_T^{\text{policy}}(\pi) \,,
\]
because for each ordering $\sigma$, one can associate a policy $\pi$, such that $\pi(S_t) = \sigma(S_t)$. But, the other direction is not true in general. Indeed, in the example of \cref{app:prob1}, the inequality is strict. This is due to the dependence between losses and availabilities. Yet, no existing efficient algorithm can handle such dependence neither for policy regret nor for ordering regret. In this appendix, we prove that when either losses or availabilities are i.i.d. with no dependence, then the two notions are equivalent.

\begin{lemma}[Stochastic losses and adversarial availabilities]
    Let $(\ell_t)_{t\geq 1}$ be an i.i.d. sequence of losses. Then, for any sequence of availability sets $(S_t)_{t\geq 1}$ such that $S_t$ may only depend on $(\ell_s)_{s\leq t-1}$, then
    \[
        \max_{\pi} R_T^\text{policy}(\pi) = \max_\sigma R_T^\text{ordering}(\sigma) \,,
    \]
    for any algorithm.
\end{lemma}

\begin{proof} The proof follows from the observation that the best policy with i.i.d. losses is to play the available action with the smallest expected loss. Such a policy corresponds to the ordering $\mu_{\sigma_i} \leq \mu_{\sigma_j}$ for all $i\leq j$. Note that this would not be true if $\ell_t$ could depend on $S_t$. 
\end{proof}

\begin{lemma}[Adversarial oblivious losses and stochastic rewards]
Let $(\ell_t)_{t\geq 1}$ be an arbitrary sequence of losses and $(S_t)_{t\geq 1}$ be a sequence of i.i.d. availability sets. Then,
\[
    \max_{\pi} R_T^\text{policy}(\pi) = \max_\sigma R_T^\text{ordering}(\sigma) \,,
\]
for any algorithm.
\end{lemma}

\begin{proof}
It is important to note here that we consider an oblivious adversary for the loss sequence $(\ell_t)$, which cannot depend on the randomness of $(S_t)$. Let $\pi: 2^{[K]} \to [K]$ be a policy, then
\[
    \E\left[ \sum_{t=1}^T \ell_t(\pi(S_t)) \right]  = \sum_{t=1}^T \sum_{S \in 2^{[K]}} \ell_t(\pi(S)) \P(S = S_t) = \sum_{S \in 2^{[K]}} p(S) \sum_{t=1}^T \ell_t(\pi(S)) 
\]
where $p(S) = \P(S_t = S)$. Thus, the best policy corresponds to the choice
\[
    \pi(S) \in \argmin_{k \in S} \sum_{t=1}^T \ell_t(k) \,.
\]
This policy corresponds to the ordering $\sum_{t=1}^T \ell_t(\sigma_i) \leq \sum_{t=1}^T \ell_t(\sigma_j)$ for $i\leq j$. 
\end{proof}

\subsection{Low internal regret $R_T^{\text{int}}$ with Stochastic losses, Adversarial availabilities does imply low ordering regret $R_T^{\text{ordering}}$}
\label{app:int_to_ord}

\begingroup
\def\thelemma{\ref{lem:implication}}
\begin{lemma}[Internal Regret Implies Ordering  (for stochastic Losses)] 
Assume that the losses $(\ell_t)_{t\geq 1}$ are i.i.d..  Then, for any sequence of availability sets $(S_t)_{t\geq 1}$ such
that $S_t$ may only depend on $(\ell_s)_{s\leq t-1}$,  for any ordering $\sigma$, we have 
\[
    \E\big[R_T^{\text{ordering}}(\sigma)\big] \leq \sum_{i=1}^K \sum_{j \in D_i} \E\big[ R_T^{\text{int}}(i\to j) \big] \,, \vspace*{-5pt}
\]
where $D_i$ is the set of arms such that $\E[\ell_t(j)] \leq \E[\ell_t(i)]$.
\end{lemma}
\addtocounter{lemma}{-1}
\endgroup

\begin{proof}
Let $\mu_k = \E\big[\ell_t(k)\big]$ for all $k \in [K]$. Let $\sigma^*$ be the best ordering such that $\mu_{\sigma_1^*} \leq \mu_{\sigma_2^*} \leq \dots \leq \mu_{\sigma_K^*}$. Note that for any ordering $\sigma$, we have $\E\big[R_T^{\text{ordering}}(\sigma)\big] \leq \E\big[R_T^{\text{ordering}}(\sigma^*)\big]$.
Thus, we can restrict ourselves to $\sigma^*$.
Denote by $k_t^* := \sigma^*(S_t)$, the best available item in $S_t$. For any $i$, we also define by $D_i := \{j \in [K]: \mu_j \leq \mu_i\}$ the items that are better than $i$. 
Then,
\begin{align*}
     \E\big[& R_T^{\text{ordering}}(\sigma)\big]  
     := \E\bigg[ \sum_{t=1}^T \ell_t(k_t) - \ell_t\big(\sigma(S_t)\big)\bigg] 
      = \E\bigg[ \sum_{t=1}^T \mu_{k_t} - \mu_{k_t^*} \bigg] \\
     & = \E\bigg[ \sum_{t=1}^T  \sum_{i=1}^K \sum_{j \in D_i}  (\mu_{i} - \mu_j) \indic\{ i = k_t, j = k_t^*\}\bigg] \quad \leftarrow \text{$k_t^*\in D_i$ because it is the best item in $S_t$} \\
     & \leq \E\bigg[ \sum_{t=1}^T  \sum_{i=1}^K \sum_{j \in D_i}  (\mu_{i} - \mu_j) \indic\{ i = k_t, j \in S_t\}\bigg] \quad \leftarrow \text{because $k_t^* \in S_t$ and $\mu_i - \mu_j\geq 0$ for any $j \in D_i$} \\
     & \leq \sum_{i = 1}^K \sum_{j \in D_i} \E\big[ R_T^{\text{int}}(i\to j) \big] \,.
\end{align*}
\end{proof}

\section{Proof of \cref{thm:internal}}

\begingroup
\def\thetheorem{\ref{thm:internal}}
\begin{theorem} 
Consider the problem of Sleeping MAB for arbitrary (adversarial) sequences of losses $\{\ell_t\}$ and availabilities $\{S_t\}$. Let $T \geq 1$ and $\smash{\eta^2 = (\log K)/ \big(2 \sum_{t=1}^T |S_t|\big)}$. Assume that $0\leq \ell_t(i)\leq 1$ for any $i \in S_t$ and $t \in [T]$. Then, \vspace*{-7pt}
\[
    \E\big[R_T^{\text{int}}(i\to j)\big] \leq 2  \sqrt{2 \log K \sum_{t=1}^T |S_t|} \leq 2 \sqrt{2 T K \log K} \,,
\]
for all $i \neq j$ in $[K]$.
\end{theorem}
\addtocounter{theorem}{-1}
\endgroup

\label{app:proof_thm}
\begin{proof}
Let $\cF_{t} := \sigma(S_1, \ell_1, k_1, \ell_1, \dots,k_{t}, S_{t+1}, \ell_{t+1})$ denotes the past randomness of the algorithm and the adversary at round $t+1$. We respectively denote by $\E_t[\,\cdot\,] := \E[\,\cdot\, | \cF_t]$ and $\P_{t}(\,\cdot\,) := \P(\,\cdot\,|\cF_t)$ the conditional expectation and probability. 
	
Note that $\tilde q_t(i \to j)$ follows the prediction of the exponentially weighted average forecaster on the losses $\hat \ell_t(i \to j)$. 
Noting that $-\eta  \hat \ell_t(i\to j) \leq 1$ for all $i \neq j$ and $t\geq 1$, and applying the upper-bound on the exponentially weighted average forecaster yields for any $i \neq j$ (see Thm. 1.5 of~\cite{hazan2021introduction}) \vspace*{-5pt}
\begin{equation}
    \label{eq:pseudoloss}
      \sum_{t=1}^T \sum_{i'\neq j'} \tilde q_t(i'\to j') \hat \ell_t(i'\to j\ ) - \sum_{t=1}^T \hat \ell_t(i\to j) 
     \leq \frac{\log (K(K-1))}{\eta} + \eta \sum_{t=1}^T \sum_{i' \neq j'} \tilde q_t(i' \to j')  \hat \ell_t(i'\to j')^2\,. 
\end{equation}
Now, we compute the expectations. Note that $S_t$, $\ell_t$ and $p_t$ are $\cF_{t-1}$-measurable by assumption. Since $k_t \in S_t$ almost surely, we have for all $j \in [K]$
\begin{align*}
  &  \E_{t-1}\big[ \hat \ell_t(i \to j)\big] 
         \stackrel{\eqref{eq:defhatellij}}{=} \E_{t-1}\Big[\sum_{k \neq i} \ell_t(k) \indic\{k = k_t, j \in S_t\} 
         + \frac{p_t(i) \ell_t(j)}{p_t(j)} \indic\{j = k_t\} + \ell_t(k_t) \indic\{j \notin S_t\}\Big] 
         \\
        & = \E_{t-1}\Big[\ell_t(k_t) \indic\{j \in S_t\}  - \ell_t(i) \indic\{ i = k_t, j\in S_t\} 
        + \frac{p_t(i) \ell_t(j)}{p_t(j)} \indic\{j = k_t\} + \ell_t(k_t) \indic\{j \notin S_t\}\Big] \\
        & = \E_{t-1}\Big[\ell_t(k_t)  - \ell_t(i) \indic\{ i = k_t, j\in S_t\} 
        + \frac{p_t(i) \ell_t(j)}{p_t(j)} \indic\{j = k_t\} \Big] \\
        & = \E_{t-1}\Big[\ell_t(k_t)  + p_t(i) (\ell_t(j)-\ell_t(i)) \indic\{j \in S_t\} \Big]  \\
        & = \E_{t-1}\Big[\ell_t(k_t)  + (\ell_t(j) -\ell_t(i)) \indic\{i = k_t, j \in S_t\} \Big]  \,.
\end{align*}
Furthermore, by definitions of $\hat \ell_t(i\to j), p_t$ and $q_t$, and denoting $\tilde Q_t = \sum_{i'\neq j'} \tilde q_t(i'\to j') \indic\{j' \in S_t\}$, we have
\begin{align*}
     \E_{t-1}\bigg[\sum_{i \neq j} \tilde q_t(i\to j)  \hat \ell_t(i\to j)\bigg] 
          & \stackrel{\eqref{eq:defhatellij}}{=}  \E_{t-1}\bigg[ \sum_{i \neq j} \tilde q_t(i\to j)  \sum_{k=1}^K p_t^{i \to j}(k) \hat \ell_t(k) \indic\{j \in S_t\} 
          +  \sum_{i \neq j} \tilde q_t(i\to j) \ell_t(k_t) \indic\{j \notin S_t\} \bigg]  \\
          & \stackrel{\eqref{eq:defqt}}{=} \E_{t-1}\bigg[ \tilde Q_t \sum_{i \neq j}  q_t(i\to j)  \sum_{k=1}^K p_t^{i \to j}(k) \hat \ell_t(k) 
          +  (1-\tilde Q_t) \ell_t(k_t) \bigg]  \\
          & \stackrel{\eqref{eq:defpt}}{=} \E_{t-1}\bigg[ \tilde Q_t   \sum_{k=1}^K p_t(k) \hat \ell_t(k) +  (1-\tilde Q_t) \ell_t(k_t) \bigg] \\
          & = \E_{t-1}\big[ \tilde Q_t    \ell_t(k_t) +  (1-\tilde Q_t) \ell_t(k_t) \big] = \E_{t-1}\big[ \ell_t(k_t) \big] \,.
\end{align*}

Therefore, the expectation of the left-hand side of~\eqref{eq:pseudoloss} equals the internal sleeping regret: 
\begin{align}
    \label{eq:lhs}
    \E\bigg[ \sum_{t=1}^T \sum_{i'\neq j'} \tilde q_t(i'\to j') \hat \ell_t(i'\to j\ ) & - \sum_{t=1}^T \hat \ell_t(i\to j)  \bigg]   = R_T^{\text{int}}(i \to j) \,. 
\end{align} 
On the other hand, \vspace*{-5pt}
\begin{align*}
    & \E_{t-1}\bigg[  \sum_{i \neq j}  \tilde q_t(i\to j)  \hat \ell_t(i\to j)^2 \bigg] \hspace{-4cm} \\
        & = \E_{t-1}\bigg[ \sum_{i \neq j}  \tilde q_t(i\to j) \Big( \sum_{k=1}^K p_t^{i \to j}(k) \hat \ell_t(k)\Big)^2 \indic\{j \in S_t\} 
        +  (1 - \tilde Q_t) \ell_t(k_t)^2   \bigg] \\
        & \leq   \E_{t-1}\bigg[ \sum_{i \neq j}  \tilde q_t(i\to j) \sum_{k=1}^K p_t^{i \to j}(k)  \hat \ell_t(k)^2 \indic\{j \in S_t\} 
        + (1 - \tilde Q_t) \ell_t(k_t)^2   \bigg] \\ 
        & \stackrel{\eqref{eq:defqt} \text{ and } \eqref{eq:defpt}}{=}   \E_{t-1}\bigg[ \tilde Q_t \sum_{k=1}^K p_t(k)  \hat \ell_t(k)^2  +  (1 - \tilde Q_t) \ell_t(k_t)^2   \bigg] \\ 
        & = \E_{t-1}\bigg[ \tilde Q_t \frac{\ell_t(k_t)^2}{p_t(k_t)}  +  (1 - \tilde Q_t) \ell_t(k_t)^2   \bigg] \\
        & = \tilde Q_t \sum_{k\in S_t} p_t(k) \frac{\ell_t(k)^2}{p_t(k)} + (1-\tilde Q_t) \E_{t-1}\big[\ell_t(k_t)\big] \\
        & \leq (|S_t|-1) \tilde Q_t  + 1 
        \leq |S_t| \,.
\end{align*}
The expectation of the right-hand-side of~\eqref{eq:pseudoloss} can thus be upper-bounded as
\[
    \eta \E\left[\sum_{t=1}^T \sum_{i \neq j}  \tilde q_t(i\to j)  \hat \ell_t(i\to j)^2 \right] \leq \eta  \sum_{t=1}^T |S_t| \,.
\]
Therefore, substituting the above inequality and~\eqref{eq:lhs}  into~\eqref{eq:pseudoloss}, and optimizing $\eta$ concludes the proof.
\end{proof}


\section{Proof of Theorem~\ref{thm:dueling}}
\label{app:db}

\begingroup
\def\thetheorem{\ref{thm:dueling}}
\begin{theorem}
Consider the problem setting of Sleeping DB defined above (\cref{sec:dueling}) and let $T \geq 1$. Then, Sparring SI-EXP3 satisfies \vspace*{-5pt}
\[
    \E[R_T^{\texttt{SI-DB}}] \leq 2 K^2 \sqrt{2 T K \log K} \,.
\]
\end{theorem}
\addtocounter{theorem}{-1}
\endgroup

\begin{proof}
Denote by $j_t^* = \argmax_{j \in A_t(i_t)} P(j,i_t)$ and by $i_t^* = \argmax_{i \in A_t(j_t)} P(i, j_t)$. Then, using that $P(i,j) = 1- P(j,i)$, we have
\begin{align}
    \E[R_T^{\texttt{SI-DB}}] & := \E\left[ \sum_{t=1}^T \frac{P(j_t^*, i_t) + P(i_t^*, j_t) - 1}{2}\right] \nonumber  \\
    & := \E\left[ \sum_{t=1}^T \frac{P(j_t,i_t)  - P(j_t,i_t^*)}{2} \right] + \E\left[ \sum_{t=1}^T \frac{P(i_t,j_t) - P(i_t,j_t^*)}{2}\right]  \,. \label{eq:regretdueling}
\end{align}
Let us focus on the first term of the r.h.s, the other one can be analysed similarly.
\begin{align*}
   P(j_t,i_t)  - P_t(j_t,i_t^*) 
        & =  \sum_{i=1}^K \sum_{i' = 1}^K  \big( P(j_t,i)  - P(j_t,i') \big) \indic\{i = i_t, i'= i_t^*\} \\
        & \leq  \sum_{i=1}^K \sum_{i' \in D_i} \big( P(j_t,i)  - P(j_t,i') \big) \indic\{i = i_t, i' \in S_t \} \,,
\end{align*}
where $D_i := \{i' \in S_t: P(i',j_t) \geq  P(i, j_t)\}$. The last inequality is because $i_t^* \in S_t \cap D_i$ and $P(j_t,i) - P(j_t,i') > 0$ for any $i' \in D_i$. Note that $D_i$ does not depend on $j_t$ because of the total ordering assumption. 
Then, taking the expectation and summing over $t$, we get
\begin{align*}
    R_T^{\text{left}} 
        & := \E\left[ \sum_{t=1}^T P(j_t,i_t)  - P_t(j_t,i_t^*) \right] \\
        & \leq \sum_{i=1}^K \sum_{i'\in D_i} \E\left[ \sum_{t=1}^T \big( P(j_t,i)  - P(j_t,i') \big) \indic\{i = i_t, i' \in S_t \}\right] \\
        & \leq \sum_{i=1}^K \sum_{i'\in D_i} \E\left[ \sum_{t=1}^T \big( \ell_t^{\text{left}}(i)  - \ell_t^{\text{left}}(i')\big) \indic\{i = i_t, i' \in S_t \}\right] \\
        & \leq \sum_{i=1}^K \sum_{i'\in D_i} 2 \sqrt{2TK \log K}  \leq 2 K^2 \sqrt{2TK \log K} \,,
\end{align*}
where the second to last inequality is by Theorem~\ref{thm:internal} by construction of $\cA^{\text{left}}$ which minimizes the internal regret. Similarly, we can show that 
\[
    R_T^{\text{right}} := \E\left[ \sum_{t=1}^T P(i_t,j_t)  - P_t(i_t^*,j_t) \right] \leq 2 K^2 \sqrt{2TK \log K} \,.
\]
Substituting both upper-bounds into~\eqref{eq:regretdueling} concludes the proof.
\end{proof}

\section{Experiments}
\label{app:expts}

\subsection{Additional experiments on sleeping multi-armed bandits}

In this section, we run some additional experiments to compare the $3$ algorithms: 
\begin{itemize}[nosep,leftmargin=*]
    \item SI-EXP3: Our proposed algorithm Internal Sleeping-EXP3 described in Sec.~\ref{sec:alg_internal}; 
    \item S-UCB: The sleeping UCB procedure proposed by \cite{kleinberg+10} for ordering regret with stochastic losses;
    \item S-EXP3: The algorithm Sleeping-EXP3G designed by \cite{saha2020improved} for ordering regret with adversarial losses and stochastic sleeping.
\end{itemize}
Again, each experiment is run 50 times and the policy regret is plotted in Figure~\ref{fig:exp2}.

\paragraph{Random environment with dependence} This setup is similar to the dependent environment of \cref{sec:expts} where the distribution of $(S_t,\ell_t)$ are uniformly sampled at the start of each run. 

More precisely. We consider the following stochastic environment with $K = 5$. The pairs $(S_t, \ell_t)$ are i.i.d. and sampled as follows. 

At the start of each run, five availability sets $\cA_1,\dots,\cA_5 \subseteq [K]$ are sampled by including independently each action with probability $1/2$. If a set contains no action, it is sampled again. 
Then, for each set $m =1,\dots,5$, a mean vector $\mu_m \in \R^K$ is uniformly sampled on $(0,1)^K$. Then, for $t=1,\dots,T$, the availability set $S_t$ is drawn uniformly from $\{\cA_1,\dots,\cA_5\}$ and the losses of each arm $k$ is sample from a Bernoulli with parameter $\mu_{m_t}(k)$, where $m_t$ is such that $S_t = \cA_{m_t}$. 

\paragraph{Random two-player zero-sum games} This setup is similar to the Rock-Paper-Scissors environment of \cref{sec:expts} but with $K = 10$ players and a random payoff matrix.  

At the start of each run, a payoff matrix $G \in \R^{K\times K}$ is randomly sampled as follows. 
For each $1\leq i<j\leq K$, 
$\smash{G_{ij} {\stackrel{\text{i.i.d.}}{\sim}} \text{Unif}\big((-1,1)\big)}$, 
$G_{ii} = 1/2$ and $G_{ji} = - G_{ij}$:
\[
    G = \left(
    \begin{smallmatrix}
        0 & G_{12} & G_{13} & \dots \\
        -G_{12} & 0 & G_{23} & \dots\\
        -G_{13} & -G_{23} & 0 & \dots \\
        \dots & \dots & \dots & 0
    \end{smallmatrix}
    \right) \,.\vspace*{-5pt}
\]
Furthermore, $4$ availability sets $(\cA_m)_{1\leq m\leq 4}$
 are randomly sampled by including each action with probability $1/2$. 
 For each $m \in [4]$, we compute $p_m \in \Delta_K$ the Nash equilibrium of the game $G$ restricted to actions in $\cA_m$.  Note that $p_m(k) = 0$ for all $k \notin \cA_m$. 
 Then, for each $t=1,\dots, T$, an availability set $S_t = \cA_{m_t}$ is uniformly sampled in $\{\cA_1,\dots,\cA_4\}$.  The algorithm is asked to choose an action $k_t \in S_t$ and receives the loss $\ell_t(k_t) \sim \cB(G_{j_t k_t})$, where $j_t$ is the action chosen by an optimal adversary that follows $p_{m_t}$.

 The optimal strategy in this case should be too also follow $k_t \sim \cA_{m_t}$ and would incur $\E[\ell_t(k_t)] = 1/2$. Figure~\ref{fig:exp2} (right) plots the cumulative pseudo-regret $R_T = \sum_{t=1}^T G_{k_t,j_t} - T/2$. As we can see, SI-EXP3 significantly outperforms S-UCB and S-EXP3. It would be worth to investigate if SI-EXP3 could be used to compute Nash equilibria in repeated two-player zero-sum games with non-available actions.

\begin{figure}
\centering
\includegraphics[width=.3\textwidth]{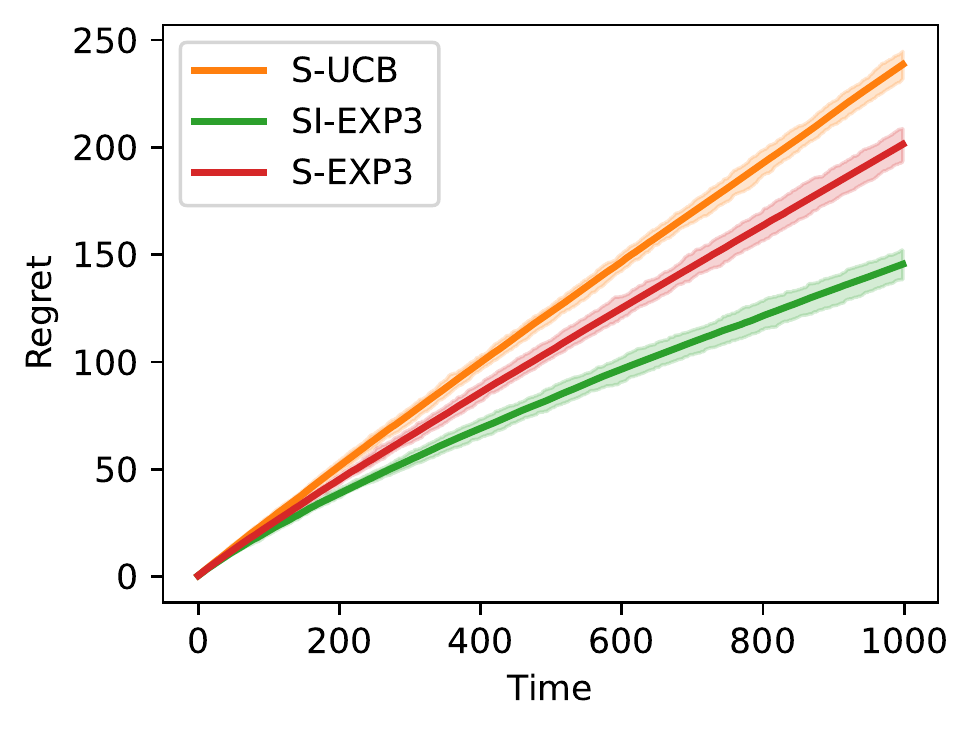}\quad 
\includegraphics[width=.3\textwidth]{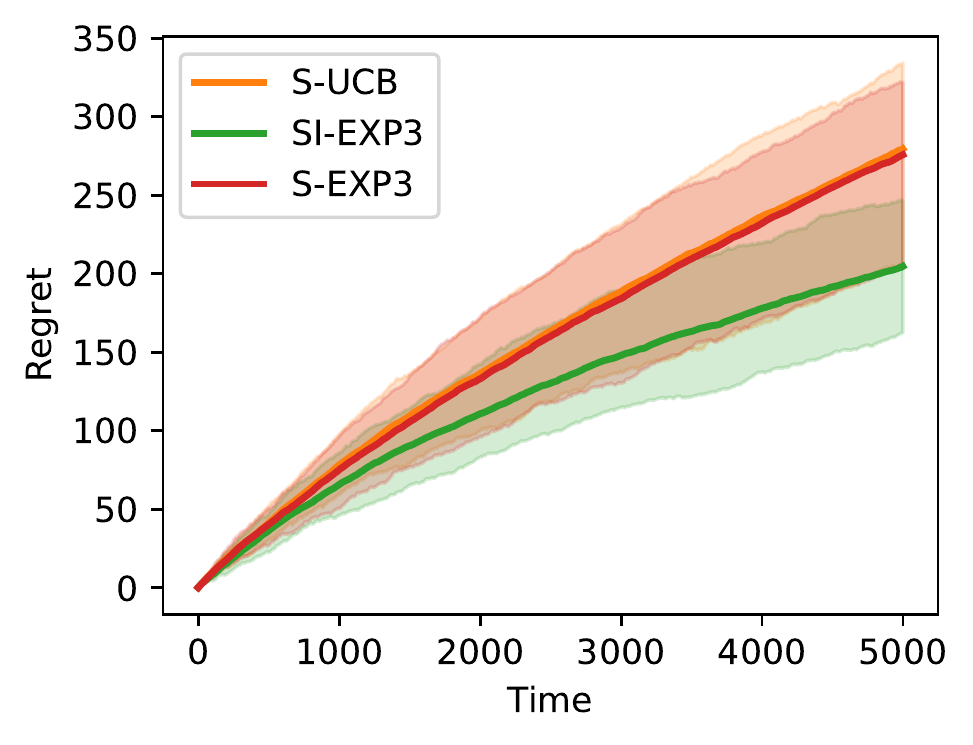}
\caption{[Left] Random environment with dependence [Right] Random games}
\label{fig:exp2}
\end{figure}

\subsection{Experiments on sleeping dueling bandits}

\paragraph{Dueling Bandits with non-repeating arms} This experimental setup is motivated by the first example in Sec. 4.1 where we want the algorithm to converge to the top 2 items (best pair). We consider utility scores $\lbrace u_1,u_2,...,u_K \rbrace$ corresponding to the $K$ arms, and the preference matrix $P$, with $P_{ij}$ defined as $P_{ij}=\frac{u_i}{u_i+u_j}$ indicating the probability of arm $i$ winning over arm $j$. We repeat this experiment for $M$ independent runs, by sampling a random utility vector at the beginning of each run. We assume that all the arms are available for the first bandit. All the arms except the one chosen by the first bandit, is available to the second bandit. Each bandit runs its own custom algorithm (which can be UCB, SI-EXP3, etc.). Finally, the winning arm is decided according to $P$ and the loss is $1$ for the bandit that chose this arm and $0$ for the other bandit. In the figures below, we plot the Internal Sleeping DB regret for choices of Sp-UCB and Sp-SIEXP3 (Sparring UCB, SI-EXP3 where both bandits internally use the UCB algorithm and the SI-EXP3 algorithm respectively). In \cref{fig:exp3} we plot Sp-SIEXP3 and Sp-UCB and we see that in this relatively simple setting Sp-UCB outperforms Sp-SIEXP3.

Note that despite its surprisingly good performance in \cref{fig:exp3}, especially for small number of arms, Sp-UCB has no theoretical guarantees for dueling bandits. It would be interesting to study whether such guarantees are possible or whether it has a linear worst-case regret. Furthermore, Sp-UCB strongly assumes a total and fixed ordering of stock performance. As we can see in the following example, Sp-SI-EXP3 works better as soon as there is some dependence between the preference matrix and the availabilities. It is also worth to emphasize that we could not compare with classical dueling bandit algorithm that are not suited for this setting.

\begin{figure}
\centering
\includegraphics[width=.3\textwidth]{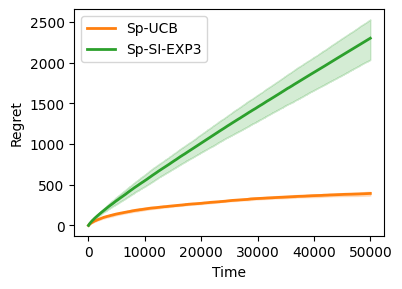}\quad 
\includegraphics[width=.3\textwidth]{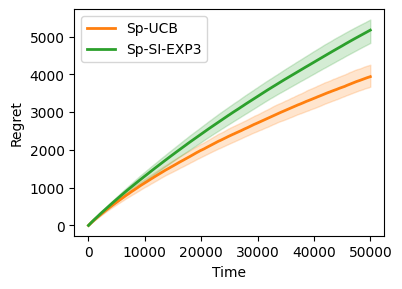}
\caption{Dueling Bandits with non-repeating arms for $K=4$ [Left] and $K=30$ [Right] respectively. ($M=5$).}
\label{fig:exp3}
\end{figure}

\begin{figure}
\centering
\includegraphics[width=.3\textwidth]{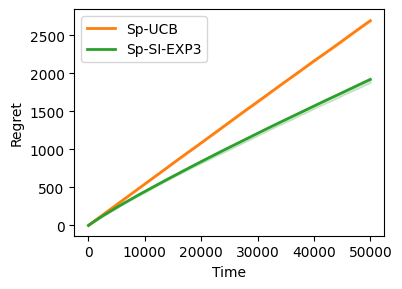}
\caption{Preference Learning with Categories where Utilities depend on Availability.}
\label{fig:exp4}
\end{figure}

\paragraph{Preference Learning with Categories} In this experimental setup we have availability dependent utility matrices. This is motivated by the following setting: if one item of a category is unavailable, the overall utility values of all items in the category goes down. In the real world, this could be in a setting where I would want to watch a season of a show only if all the seasons are available. Concretely, we have $K$ different availability sets, where $\cA_i$ has all items available except $i$. We also have $K$ utility vectors: $\lbrace  u_1,u_2,...,u_K \rbrace$. At each turn we randomly choose $r \in \lbrace 1,2,...,K \rbrace$ and select $\cA_r$ and $u_r$. Similar to the previous setting, the first bandit chooses an available item and the second bandit chooses an available item except the one chosen by the first bandit. In \cref{fig:exp4}, we choose the utility vectors as $\lbrace (1,2,...,K), (K,1,2,..., K-1), ..., (2,3,..., K,1)\rbrace$ and we see that Sp-SIEXP3 significantly outperforms Sp-UCB.

}

\end{document}